\documentclass[conference]{IEEEtran}
\IEEEoverridecommandlockouts
\usepackage{cite}
\usepackage{amsmath,amssymb,amsfonts}
\usepackage{graphicx}
\usepackage{textcomp}
\usepackage{xcolor}
\def\BibTeX{{\rm B\kern-.05em{\sc i\kern-.025em b}\kern-.08em
    T\kern-.1667em\lower.7ex\hbox{E}\kern-.125emX}}

\usepackage{hyperref}
\usepackage{cleveref}
\usepackage[shortlabels]{enumitem}
\usepackage{xspace}
\usepackage{hyphenat}
\usepackage{todonotes}
\usepackage{booktabs}
\usepackage{mathtools}
\usepackage{balance}
\usepackage{float}
\usepackage{layouts}
\usepackage[font=small,labelfont=bf,tableposition=top]{caption}
\DeclareCaptionLabelFormat{andtable}{#1~#2  \&  \tablename~\thetable}

\usepackage{amsthm}
\usepackage{array,multirow}
\usepackage{algorithm}
\usepackage{algorithmicx}
\usepackage[noend]{algpseudocode}
\newcolumntype{C}[1]{>{\centering\let\newline\\\arraybackslash\hspace{0pt}}m{#1}}

\algrenewcommand\algorithmicindent{0.7em}

\newcommand{\ourcompressor}{NeaTS\xspace}
\newcommand{\ourcompressorlossy}{\mbox{NeaTS-L}\xspace}

\newcommand{\rank}{\textsf{rank}\xspace}
\newcommand{\select}{\textsf{select}\xspace}

\newtheorem{corollary}{Corollary}
\newtheorem{theorem}{Theorem}
\newtheorem{definition}{Definition}

\newcommand{\slp}{\ensuremath{m}}
\newcommand{\ic}{\ensuremath{b}}
\newcommand{\funset}{\ensuremath{\mathcal{F}}}
\newcommand{\errset}{\ensuremath{\mathcal{E}}}
\newcommand{\dist}{\ensuremath{\mathit{distance}}}
\newcommand{\prev}{\ensuremath{\mathit{previous}}}

\newcommand{\vstart}{\ensuremath{S}}
\newcommand{\vepsilon}{\ensuremath{B}}
\newcommand{\vcorrections}{\ensuremath{C}}
\newcommand{\voffcorr}{\ensuremath{O}}
\newcommand{\vkinds}{\ensuremath{K}}
\newcommand{\vparams}{\ensuremath{P}}

\newcommand{\brotli}{Brotli\xspace}
\newcommand{\chimp}{Chimp\xspace}
\newcommand{\chimpo}{Chimp128\xspace}
\newcommand{\DAC}{DAC\xspace}
\newcommand{\leco}{LeCo\xspace}
\newcommand{\gorilla}{Gorilla\xspace}
\newcommand{\tsxor}{TSXor\xspace}
\newcommand{\xz}{Xz\xspace}
\newcommand{\lz}{Lz4\xspace}
\newcommand{\snappy}{Snappy\xspace}
\newcommand{\zstd}{Zstd\xspace}
\newcommand{\alp}{ALP\xspace}

\begin{document}

\title{Learned Compression of Nonlinear Time Series With~Random~Access\thanks{This work has been accepted for publication in Proceedings of the 41st IEEE International Conference on Data Engineering (ICDE 2025)}}

\author{\IEEEauthorblockN{Andrea Guerra*}
\IEEEauthorblockA{\textit{University of Pisa} \\
Pisa, Italy}
\and
\IEEEauthorblockN{Giorgio Vinciguerra*}
\IEEEauthorblockA{\textit{University of Pisa} \\
Pisa, Italy
}
\and
\IEEEauthorblockN{Antonio Boffa\textsuperscript{†}}
\IEEEauthorblockA{\textit{EPFL} \\
Lausanne, Switzerland
}
\and
\IEEEauthorblockN{Paolo Ferragina\textsuperscript{†}}
\IEEEauthorblockA{\textit{Sant'Anna School for Advanced Studies} \\
Pisa, Italy}}

\maketitle
\begingroup\renewcommand\thefootnote{*}
\footnotetext{Equal contribution. Correspondence to: andrea.guerra@phd.unipi.it, giorgio.vinciguerra@unipi.it}
\begingroup\renewcommand\thefootnote{†}
\footnotetext{Work done while the author was at the University of Pisa.}
\endgroup
\endgroup

\begin{abstract}
Time series play a crucial role in many fields, including finance, healthcare, industry, and environmental monitoring. The storage and retrieval of time series can be challenging due to their unstoppable growth. In fact, these applications often sacrifice precious historical data to make room for new data.

General-purpose compressors like \xz and \zstd can mitigate this problem with their good compression ratios, but they lack efficient random access on compressed data, thus preventing real-time analyses. Ad-hoc streaming solutions, instead, typically optimise only for compression and decompression speed, while giving up compression effectiveness and random access functionality.
Furthermore, all these methods lack awareness of certain special regularities of time series, whose trends over time can often be described by some linear and nonlinear functions.

To address these issues, we introduce \ourcompressor, a randomly\hyp{}accessible compression scheme that approximates the time series with a sequence of nonlinear functions of different kinds and shapes, carefully selected and placed by a partitioning algorithm to minimise the space.
The approximation residuals are bounded, which allows storing them in little space and thus recovering the original data losslessly, or simply discarding them to obtain a lossy time series representation with maximum error guarantees.

Our experiments show that \ourcompressor improves the compression ratio of the state-of-the-art lossy compressors that use linear or nonlinear functions (or both) by up to 14\%.
Compared to lossless compressors, \ourcompressor emerges as the only approach to date providing, simultaneously, compression ratios close to or better than the best existing compressors, a much faster decompression speed, and orders of magnitude more efficient random access, thus enabling the storage and real-time analysis of massive and ever-growing amounts of (historical) time series data.
\end{abstract}

\section{Introduction}\label{sec:intro}

Time series are pervasive across a multitude of fields, including finance, healthcare, industry, and environmental monitoring. 
These sorted sequences of time-stamped data points represent a wide variety of dynamic phenomena, from market prices to patient vitals and sensor readings, and they have become invaluable for decision-making, trend analysis, and forecasting.

Unsurprisingly, the efficient storage, transmission, and analysis of time series have become more and more challenging as their volume has grown exponentially~\cite{JensenPT17survey,Aggarwal2013book}, leading to the development of numerous time series databases~\cite{Wang:2023,influxdata,opentsdb,prometheus,timescale}.

Data compression is the key strategy to lower the cost of time series storage and transmission~\cite{Chiarot:2022survey}.
The easiest way to approach it is to use one of the off-the-shelf general-purpose compressors (such as \brotli~\cite{Alakuijala:2018}, \zstd~\cite{Collet:2016}, \xz~\cite{Xz}, \lz~\cite{Collet:2013}, \snappy~\cite{Snappy}, etc.).
These tools are capable of achieving commendable compression ratios, but they require a substantial computational overhead, both in terms of CPU and memory usage, which often makes them unsuitable on hardware- and energy-constrained devices such as smartphones, smart wearable, IoT or edge devices.

Motivated by this shortcoming, several new special-purpose compressors have been developed for time series, often reducing the computational overhead at the expense of lower compression ratios.
Most notably, several works~\cite{Pelkonen:2015,Liakos:2022,Li:2023,ALP} have shown how to compress and decompress floating-point time series data much faster than general-purpose compressors, enabling both high ingestion rates and efficient scans.
However, not much attention has been given to the design and the evaluation of the random access operation to single values of the time series~\cite{Vestergaard:2021,Brandon:2021}, even in benchmarking studies~\cite{FCBench}.
{ This is quite surprising given that the most fundamental queries in time series databases ultimately rely on accessing data within a specific time interval~\cite{KhelifatiKDDC23,hao2021ts}, which from a compressed storage perspective boils down to combining a random access operation (to retrieve the first data point) with a scan (to retrieve the subsequent data points within the interval).
However, providing efficient random access is challenging, and it often conflicts with achieving good compression ratios.
}

Furthermore, none of the above compressors can harness a key peculiarity of time series data: its trends over time can often be described by some linear and nonlinear functions~\cite{TurkmanBook,FanBook}.
Indeed, although there is a rich literature on approximating and indexing a time series via linear~\cite{KeoghCHP01segmenting,Liu:2008,Xie:2014}, polynomial and other functions~\cite{Fuchs:2010,Eichinger:2015,Xu:2012,Qi:2015}, or via Fourier and wavelet transforms~\cite{Keogh:2001}, all these methods are lossy~\cite{Chiarot:2022survey} and thus inapplicable in cases where we need to reconstruct the original data for accurate analyses.
A step in this direction has been made by some  learned compressors that are not specifically designed for time series~\cite{Ao:2011,Boffa:2022talg,leco:2024}. But these approaches either exploit linear functions only~\cite{Ao:2011,Boffa:2022talg} or use sub-optimal partitioning algorithms and non-error-bounded approximations~\cite{Ao:2011,leco:2024}, { so they fall short of reaching the best possible compression efficacy.}

\smallskip\noindent\textbf{Our contribution.} We contribute to the long line of research on time series compression as follows:
\begin{itemize}[leftmargin=*]
    \item We show how to compute piecewise approximations using several kinds of \emph{nonlinear} functions (such as quadratic, radical, exponential, logarithmic, and Gaussian) under a given error bound~$\varepsilon$ optimally, i.e. in linear time and with the guarantee that the number of pieces is minimised. This generalises the classic algorithm to compute piecewise \emph{linear} approximations~\cite{ORourke:1981}.

    \item We introduce an algorithm to partition a time series into variable\hyp{}sized fragments, each associated with a \emph{different} nonlinear approximation, so that the space of the output is minimised. This generalises a previous result for increasing linear functions only~\cite{Boffa:2022talg}.

    \item By combining the above two results with proper succinct data structures, we design \ourcompressor, a new { randomly-accessible} compression scheme that approximates the time series with a sequence of nonlinear functions of different kinds and shapes. The residuals of the approximation are bounded, which allows storing them in little space and thus recovering the original data losslessly, or discarding them to obtain a lossy time series representation with maximum error guarantees.
    \Cref{fig:layout} shows an example of \ourcompressor.

    \item We conduct a thorough experimental evaluation on 16 real-world time series, whose size ranges from thousands to hundreds of millions of data points, comparing our \ourcompressor against 2 lossy compressors, as well as 5 general-purpose and 7 special-purpose lossless compressors, including the recent ALP~\cite{ALP} and \leco~\cite{leco:2024}. Our results show that the lossy version of \ourcompressor improves uniformly the compression ratio of previous error-bounded approximations based on linear or nonlinear functions (or both), with an improvement of up to~14\%. 
    Compared to lossless compressors, \ourcompressor emerges as the only approach to date delivering, simultaneously, compression ratios close to or better than the existing compressors (i.e. the best compression ratio among the special-purpose compressor on 14/16 datasets, and the best overall on 4/16 dataset), a much faster decompression speed, and up to 3 orders of magnitude more efficient random access.
    No other compressor to date can achieve such a good performance in one of these factors without significantly sacrificing others.
    { We finally show that \ourcompressor delivers superior performance across range queries of different sizes, thus benefiting the wide variety of queries in time series databases that access data within specific time intervals.}

\end{itemize}

\smallskip\noindent\textbf{Outline.}
\Cref{sec:background} gives the background and definitions.
\Cref{sec:our-compressor} introduces our \ourcompressor. 
\Cref{sec:experiments} presents our experimental results.
\Cref{sec:related} discusses related work.
\Cref{sec:conclusions} concludes the paper and suggests some open problems.

\section{Background}\label{sec:background}

We now provide some background information, starting with a definition of the data we compress.

\begin{definition}[Time series]\label{def:time-series}
A \emph{time series} $T$ is a sequence of $n$ data points of the form $(x_k, y_k)$, where $x_k \in \mathbb{N}$ is the timestamp, and $y_k \in \mathbb{Z}$ is the value associated with it, ordered increasingly by time, i.e. $T = [(x_1, y_1), (x_2,y_2), \dots, (x_n,y_n)]$ where $x_1 < x_2 < \dots < x_n$. A \emph{fragment} of $T$ is a subsequence $T[i,j] = [(x_i, y_i), \dots, (x_j,y_j)]$ for any two indexes $i,j$ such that $1 \leq i \leq j \leq n$.
\end{definition}

We require the values to be integers, which is common in practice. In fact, single/double-precision floating-point values can be interpreted as 32/64-bit integers, or better, since the values in real-world time series typically have a fixed number $x$ of significant digits after the decimal point, we can multiply them by the constant $10^{x}$ and turn them into integers~\cite{Brandon:2021}. 

\begin{figure}[t]
\centering
\includegraphics{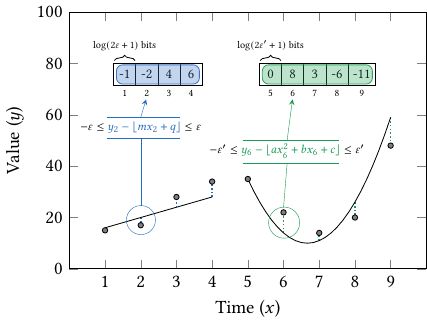}
\caption{\looseness=-1 \ourcompressor represents fragments of the time series via linear or nonlinear functions learned from the data. The residuals of the approximation are bounded by a value $\varepsilon$ so, if lossless compression is needed, we store them in packed arrays (shown on top).}
\label{fig:layout}
\end{figure}

We focus on functional approximations for time series but, unlike known approaches~\cite{Chiarot:2022survey}, we use them to design not only a lossy compressor but also a lossless one.

To illustrate, consider a time series $T$ on a Cartesian plane where the horizontal axis represents time and the vertical axis represents values.
Any function~$f$ passing through all the data points $(x_1,y_1),\dots,(x_n,y_n)$ is a lossless encoding of $T$ because, for a given timestamp $x_k$, we can recover the corresponding value as $y_k=f(x_k)$, thus requiring us to store only the (parameters of the) function $f$.
However, the number of parameters of a function passing exactly through all the data points could be so large to result in no compression (consider e.g. a polynomial interpolation of the $n$ data points, which generally requires storing $n$ coefficients). 
Therefore, we allow $f$ to make some ``errors'' but in a controlled way, namely, we bound the infinity norm of the errors.

\begin{definition}[$\varepsilon$\hyp{}approximations]\label{def:eps-approx}
Let $\varepsilon \geq 0$ be an integer. A function $f$ is said to be an $\varepsilon$-\emph{approximation} of a time series $T$ (or fragment $T[i,j]$) if we have $|f(x_k) - y_k | \leq \varepsilon$ for every data point $(x_k, y_k)$ in $T$ (resp. $T[i,j]$).
\end{definition}

A function $f$ defined in this way is a lossy representation of~$T$. To make it lossless, as observed in~\cite{Boffa:2022talg}, it is sufficient to also store the ``corrections'' (i.e. residuals) $c_k = y_k - \lfloor f(x_k) \rfloor$ in $\lceil \log (2\varepsilon + 1) \rceil$ bits each, which thus allows recovering $y_k$ as $\lfloor f(x_k) \rfloor + c_k$.
Intuitively, the smaller the value of $\varepsilon$, the less space is needed to store the corrections; however, at the same time, more space may be required to store the parameters of~$f$, due to it being more complex to better fit the data points.

In this scenario, there are two issues to solve. The first is how to compute a function that $\varepsilon$\hyp{}approximates (a fragment of) $T$.
The second is how to choose $\varepsilon$ so that the storage of both the function $f$ and the corrections takes minimal space, possibly using a different $\varepsilon$-value for different fragments of~$T$.
In the case $f$ is a linear function, these two issues have already been solved~\cite{ORourke:1981,Boffa:2022talg}.

\label{ssec:computing-linear}
We now recall how to compute a linear  $\varepsilon$\hyp{}approximation, which will be the starting point of our extension to nonlinear functions.
Given an integer $\varepsilon \geq 0$ and an index $i$ into a time series $T$, O'Rourke's algorithm~\cite{ORourke:1981} finds the longest fragment $T[i,j]$ that admits a linear $\varepsilon$\hyp{}approximation $f(x) = \slp x + \ic$ in optimal $O(j-i)$~time.\footnote{More recent papers address this same problem~\cite{KeoghCHP01segmenting,Dalai:2006,Elmeleegy:2009,Liu:2008,Xie:2014}, sometimes proposing algorithms that are equivalent, or sub-optimal in terms of time complexity, or that find shorter fragments compared to the earlier algorithm by O'Rourke~\cite{ORourke:1981} we consider~here.}
The algorithm works by maintaining a set~$P$ of feasible coefficients in the 2D space with $\slp$ on the horizontal axis and $\ic$ on the vertical axis, and by processing the data points in $T[i,j]$ left-to-right by shrinking $|P|$ at each data point. Regarding $P$, notice that, according to~\Cref{def:eps-approx}, the linear function $f$ must satisfy the inequality $|f(x_k) - y_k| = |\slp x_k + \ic - y_k| \leq \varepsilon$ for every $k=i,\dots,j$, which can be rewritten as $\ic \leq (-x_k) \slp + y_k + \varepsilon$ and $\ic \geq (-x_k) \slp + y_k - \varepsilon$.
Therefore, $P$ is a convex polygon in the 2D space defined by the intersection of the $2(j-i+1)$ half-planes arising from these two inequalities for $k=i,\dots,j$.

When adding the $(j+1)$th data point would cause $P$ to be empty (i.e. the fragment $T[i,j]$ cannot be made any longer), the algorithm stops and picks a pair $(\slp,\ic) \in P$ as the coefficients of the linear function $f$ that $\varepsilon$\hyp{}approximates the fragment $T[i,j]$.

Let us call $P_k$ the polygon at a generic step $k$.
Crucial for the time-efficiency of the algorithm is the fact that each edge of $P_k$ has a slope that lies in the range $[-x_k,0)$, which is easy to see since each half-plane defining $P_k$ has slope of the form $-x_l$ (due to the above inequalities) and that $0 < x_l < x_k$ for $l < k$ (due to~\Cref{def:time-series}).
This fact allows updating $P_k$ with the inequalities arising from the $(k+1)$th data point in constant amortised time.

We refer the reader to the seminal paper~\cite{ORourke:1981} for more details and anticipate that we will generalise this algorithm to work with several kinds of nonlinear functions.

\section{The \protect\ourcompressor compressor}\label{sec:our-compressor}
We now introduce our compression scheme for time series. We call it \ourcompressor (\underline{N}onlinear \underline{e}rror-bounded \underline{a}pproximation for \underline{T}ime \underline{S}eries) since it exploits the potential of nonlinear functions to compress time series.
We design it in three steps.

\begin{enumerate}[leftmargin=*]
    \item We describe how to compute an $\varepsilon$\hyp{}approximation for a fragment of the time series using several kinds of nonlinear functions, such as quadratic, radical, exponential, logarithmic, and Gaussian functions  (\Cref{ssec:nonlinear}).

    \item We introduce an algorithm to divide the input time series into variable-sized fragments, each with an associated nonlinear approximation with a different $\varepsilon$-value, so that the space of the compressed output is minimised. While our focus is on lossless compression, we also show how to adapt this algorithm to produce a (lossy) piecewise nonlinear $\varepsilon$\hyp{}approximation of a time series in linear time (\Cref{ssec:space-min}).

    \item We show how to support random access to individual values of the time series by combining the compressed output with proper succinct data structures (\Cref{sec:random-access}).
\end{enumerate}

\subsection{Computing a nonlinear \texorpdfstring{$\varepsilon$}{ε}-approximation} \label{ssec:nonlinear}
Linear functions have surely been the most widely used representations for time series due to their simplicity~\cite{KeoghCHP01segmenting,Dalai:2006,Elmeleegy:2009,Liu:2008,Xie:2014}. Nonetheless, they may not always be the best choice to approximate a time series because of the possible presence of nonlinear patterns in real-world data~\cite{TurkmanBook,FanBook}.

Let us be given an integer $\varepsilon \geq 0$ and an index $i$ into a time series~$T$. We now show how to find the longest fragment $T[i,j]$ that admits an $\varepsilon$\hyp{}approximation for some kinds of nonlinear functions with two parameters $\theta_1$ and $\theta_2$.
We do so by generalising the classic algorithm by O'Rourke~\cite{ORourke:1981} for linear functions (recalled in~\Cref{ssec:computing-linear}).

We start with exponential functions of the form $f(x) = \theta_2\, e^{\theta_1 x}$. According to~\Cref{def:eps-approx}, we must ensure that $f$ satisfies the inequality $|f(x_k)-y_k| = |\theta_2 e^{\theta_1 x_k}-y_k| \leq \varepsilon$ for every data point $k=i,\dots,j$, which can be rewritten as
\begin{equation*}
\begin{aligned}
    &\ln \theta_2 \leq (-x_k) \theta_1 + \ln(y_k + \varepsilon) \\
    &\ln \theta_2 \geq (-x_k) \theta_1 + \ln(y_k - \varepsilon),
\end{aligned}
\end{equation*}
under the assumption $y_k - \varepsilon > 0$.\footnote{If not satisfied, just add $\varepsilon+1-\min_k y_k$ to all $y_k$s in the time series.}

The above inequalities define a pair of half-planes in the Cartesian plane with $\theta_1$ on the horizontal axis and $\ln \theta_2$ on the vertical axis. Therefore, their intersection for $k=i,\dots,j$ originates a convex polygon of feasible parameters for the exponential function~$f$.
Since the slope of each polygon edge is $-x_k$, a direct reduction to the algorithm by O'Rourke gives an optimal $O(j-i)$-time algorithm to compute an exponential $\varepsilon$\hyp{}approximation of $T[i,j]$.

In general, we can show the following.
\begin{theorem}\label{thm:nonlinear}
Let $T=[(x_1,y_1),\dots,(x_n,y_n)]$ be a time series, let $\varepsilon \geq 0$ be an integer, and let $f$ be a function with two parameters $\theta_1$ and $\theta_2$. 
If, for any $k$, the inequalities $-\varepsilon \leq f(x_k)-y_k \leq \varepsilon$ can be transformed into inequalities of the form $\alpha_k \leq t_k \slp + \ic \leq \omega_k$, where:
\begin{enumerate}[leftmargin=*]
    \item $\alpha_k$, $t_k$, and $\omega_k$ can be computed in constant time from $\varepsilon$ and $T[k]$;

    \item $\slp$ and $\ic$ are derived from $\theta_1$ and $\theta_2$, respectively, via a change of variable, i.e. $\slp=\phi(\theta_1)$ and $\ic=\psi(\theta_2)$ for some invertible functions $\phi$ and $\psi$;

    \item $t_k$ is a positive increasing function of $x_k$.
\end{enumerate}
Then, for any $i$, we can compute the longest fragment $T[i,j]$ that admits an $\varepsilon$\hyp{}approximation $f$ in optimal $O(j-i)$ time. 
\end{theorem}
\begin{proof}
We reduce the computation of the feasible parameters $\theta_1$ and $\theta_2$ of the (possibly nonlinear) function $f$ to the intersection of half-planes in the 2D space with $\slp$ on the horizontal axis and $\ic$ on the vertical axis, for which we can use the algorithm of O'Rourke~\cite{ORourke:1981}.
Using the same notation as in~\cite{ORourke:1981}, let us rewrite the transformed inequalities as $\ic \geq (-t_k)\slp + \alpha_k$ and $\ic \leq (-t_k)\slp + \omega_k$.
These inequalities clearly represent half-planes in that 2D space.

Now, let $P_k$ be the convex polygon resulting from the intersection of these inequalities where the subscript ranges as $i,i+1,\dots,k$.
Analogously to~\cite[Lemma~1]{ORourke:1981}, we need to establish the property that the slope of each edge of $P_k$ belongs to $[-t_k, 0)$.
The polygon is specified by edges of the form $\ic=(-t_l)\slp + \alpha_l$ and $\ic=(-t_l)\slp + \omega_l$ for $l = i, \dots, k$. By assumption (3), the value $t_l$ is the result of applying a positive increasing function to $x_l$. This, combined with the fact that $0 < x_l < x_k$ (by~\Cref{def:time-series}), implies that each edge of $P_k$ has slope $-t_l \geq -t_k$ and that $-t_l < 0$, thus each slope belongs to $[-t_k, 0)$.
This property is enough to guarantee the correctness and time complexity of the algorithm that maintains $P_k$ (cf. proofs of~\cite[Theorems 1 and 2]{ORourke:1981}).
We thus conclude by observing that once the next data point $T[j+1]$ causes the polygon to be empty, we choose a pair $(\slp,\ic) \in P_j$ and return $\theta_1=\phi^{-1}(\slp)$ and $\theta_2=\psi^{-1}(\ic)$ as the parameters of $f$.
\end{proof}

Taking again as an example exponential functions of the form $f(x) = \theta_2 e^{\theta_1 x}$, we first transform the inequalities $-\varepsilon \leq f(x_k) - y_k \leq \varepsilon$ via simple algebraic manipulations to
\begin{equation*}
  \underbrace{\ln(y_k - \varepsilon)}_{\alpha_k}
  \leq \underbrace{x_k}_{t_k} \underbrace{\theta_1}_{\slp} + \underbrace{\ln \theta_2}_\ic
  \leq \underbrace{\ln(y_k+\varepsilon)}_{\omega_k},
\end{equation*}
and then we apply~\Cref{thm:nonlinear}, which gives the desired exponential $\varepsilon$\hyp{}approximation for a fragment $T[i,j]$ in optimal $O(j-i)$ time.\footnote{The logarithm and other operations can be computed in constant time with mild assumptions on the model of computation~\cite{grandjean2023arithmetic}.}
\Cref{tab:examples} shows other examples with linear, exponential, power, logarithmic, and radical functions.\footnote{It is straightforward (and sometimes useful to better approximate the data) to compute a function whose graph is horizontally shifted to the first timestamp $x_i$ of $T[i,j]$: we simply store $x_i$, subtract it from the timestamps in $T[i,j]$, then we apply \Cref{thm:nonlinear} to compute a function $g$ with domain $[0,x_j-x_i]$ and output $f(x)=g(x-x_i)$.}

\begin{table}
\caption{Some examples of two-parameter functions $f$ that we can use in~\Cref{thm:nonlinear}, together with the terms $\slp,\ic,t_k,\alpha_k$ and $\omega_k$ defining the transformed inequalities.}
\label{tab:examples}
\centering
\renewcommand{\arraystretch}{1}
\begin{tabular}{cccccc}
  \toprule
  $f(x)$ 
    & $\slp$
    & $\ic$
    & $t_k$ 
    & $\alpha_k$ 
    & $\omega_k$ \\
  \midrule
  $\theta_2 e^{\theta_1 x}$ 
    & $\theta_1$
    & $\ln\theta_2$
    & $x_k$
    & $\ln(y_k - \varepsilon)$
    & $\ln(y_k + \varepsilon)$ \\
  $\theta_2 x^{\theta_1}$
    & $\theta_1$
    & $\ln\theta_2$
    & $\ln x_k$
    & $\ln(y_k - \varepsilon)$
    & $\ln(y_k + \varepsilon)$ \\
  $\ln(\theta_2 x^{\theta_1})$
    & $\theta_1$
    & $\ln\theta_2$
    & $\ln x_k$ 
    & $y_k-\varepsilon$ 
    & $y_k+\varepsilon$ \\
  $ \theta_1 x + \theta_2$ 
    & $\theta_1$
    & $\theta_2$
    & $x_k$
    & $y_k - \varepsilon$
    & $y_k + \varepsilon$ \\
  $\theta_1\sqrt{x}+\theta_2$ 
    & $\theta_1$
    & $\theta_2$
    & $\sqrt{x_k}$
    & $y_k-\varepsilon$
    & $y_k+\varepsilon$ \\
  $\theta_1 x^2 + \theta_2$
    & $\theta_1$
    & $\theta_2$
    & $x_k^2$
    & $y_k-\varepsilon$
    & $y_k+\varepsilon$ \\
  $\theta_1 x^2 + \theta_2x$
    & $\theta_1$
    & $\theta_2$
    & $x_k$
    & $(y_k-\varepsilon)/x_k$
    & $(y_k+\varepsilon)/x_k$ \\
  $\theta_1 x^3 + \theta_2 x$
    & $\theta_1$
    & $\theta_2$
    & $x_k^2$
    & $(y_k-\varepsilon)/x_k$
    & $(y_k+\varepsilon)/x_k$ \\
  $\theta_1 x^3 + \theta_2 x^2$
    & $\theta_1$
    & $\theta_2$
    & $x_k$
    & $(y_k-\varepsilon)/x_k^2$
    & $(y_k+\varepsilon)/x_k^2$ \\ 
  \bottomrule
\end{tabular}
\end{table}

In some cases, we can use~\Cref{thm:nonlinear} even for functions $f$ with three parameters, provided that we add some constraints to reduce the number of \emph{free} parameters to two, since otherwise the set of feasible parameters for $f$ becomes a polyhedron $P_k$ in a 3D space (i.e. one dimension for each parameter), which we cannot handle in linear time~\cite{PreparataM79,Overmars83,Xu:2012,Qi:2015}. 

Take, for example, quadratic functions of the form $f(x)=\theta_1 x^2 + \theta_2 x + \theta_3$. By forcing the function to pass through the first data point~$T[i]$, i.e. by setting $f(x_i)=y_i$ and thus fixing $\theta_3=y_i-\theta_1 x_i^2 - \theta_2 x_i$ (which we store explicitly), we can transform the inequalities ${-\varepsilon \leq f(x_k)-y_k \leq \varepsilon}$ via simple algebraic manipulations to
\begin{equation*}
    \newcommand{\vphantomfrac}{\vphantom{\frac{x_k^2}{x_k}}}
    \underbrace{\frac{y_k-y_i-\varepsilon}{x_k-x_i}}_{\alpha_k} 
    \leq \underbrace{(x_k+x_i)\vphantomfrac}_{t_k}%
         \underbrace{\theta_1\vphantomfrac}_m + \underbrace{\theta_2\vphantomfrac}_b
    \leq \underbrace{\frac{y_k-y_i+\varepsilon}{x_k-x_i}}_{\omega_k},
\end{equation*}
and then we apply~\Cref{thm:nonlinear}.
A similar derivation can be done for Gaussian-like functions of the form $f(x)=e^{\theta_1 x^2 + \theta_2 x + \theta_3}$.

We conclude this section by observing that a repeated application of~\Cref{thm:nonlinear} from $T[1]$ to $T[n]$ allows partitioning $T$ into the longest fragments associated with an $\varepsilon$\hyp{}approximation, thus giving the following result.

\begin{corollary}\label{cor:lossy-mono-function}
   Given a time series $T=[(x_1,y_1),\dots,(x_n,y_n)]$, a value $\varepsilon \geq 0$, and a function $f$ satisfying the assumptions of \Cref{thm:nonlinear}, we can compute a piecewise $\varepsilon$\hyp{}approximation of~$\,T$ with the smallest number of functions of the $f$-kind in $O(n)$~time.
\end{corollary}

\Cref{cor:lossy-mono-function} directly yields a lossy error-bounded (in terms of infinity norm) representation of $T$. 
As discussed in~\Cref{sec:background}, this can be made lossless by storing the corrections $y_k-\lfloor f(x_k) \rfloor$ in $\lceil \log(2\varepsilon+1)\rceil$~bits each.
In the next section, we describe a more powerful partitioning algorithm to orchestrate different types of nonlinear functions and error bounds.

\subsection{Partitioning a time series with nonlinear \texorpdfstring{$\varepsilon$}{ε}-approximations}\label{ssec:space-min}

Let us be given a set $\funset$ of functions that satisfy the assumptions of~\Cref{thm:nonlinear}, and a set $\errset$ of error bounds. We now turn our attention to the problem of partitioning a time series $T$ into fragments, each $\varepsilon$\hyp{}approximated (with $\varepsilon \in \errset$) by a function from $\funset$, with the goal of minimising the overall space of the lossless representation of~$T$, which is given by the storage of the corrections and the parameters of the functions.

At a high level, our approach computes, for each $f \in \funset$ and each $\varepsilon \in \errset$, the piecewise $\varepsilon$\hyp{}approximation of $T$ composed of functions of the $f$-kind, and then produces the desired partition of $T$ by stitching together properly-chosen fragments (possibly adjusting their start and end points) taken from the $|\funset|\cdot|\errset|$  different piecewise approximations of $T$. This generalises a previous result for increasing linear functions only~\cite{Boffa:2022talg}

More in detail, we define a graph $\mathcal G$ with one node for each data point in $T$, plus one sink node denoting the end of the time series.
Each fragment $T[i,j-1]$ that is $\varepsilon$\hyp{}approximated by a function $f\in \funset$ produces an edge $(i,j)$ of $\mathcal G$ whose weight $w_{f,\varepsilon}(i,j)$ is defined as the bit-size of the compression of $T[i,j-1]$ via $f$ and the $j-i$ corrections stored in ${\lceil\log(2\varepsilon+1)\rceil}$-bits each, i.e. $w_{f,\varepsilon}(i,j) = (j-i) \lceil\log(2\varepsilon+1)\rceil + \kappa_f$, where $\kappa_f$ is the space in bits taken by the parameters of $f$ (plus some small metadata, such as the function kind, encoded as an index from $\{1,\dots,|\funset|\}$).
Moreover, since $f$ is also an $\varepsilon$\hyp{}approximation of any prefix and suffix of $T[i,j-1]$, other than the edge $(i,j)$ we add to $\mathcal G$ also the prefix edge $(i,k)$ and the suffix edge $(k,j)$, for all $k=i,\dots,j-1$~\cite{Boffa:2022talg}.
It is not difficult to conclude that the shortest path from node $1$ to node $n + 1$ gives the desired partition of $T$.

It is well-known that, in the case of a directed acyclic graph (like $\mathcal G$), the shortest path can be computed by taking the nodes in order $1,\dots,n$ and relaxing their outgoing edges, i.e. checking whether these edges can improve the shortest path found so far~\cite{Cormen:2009book}. Furthermore, generalising what has been done in~\cite{Boffa:2022talg}, instead of precomputing all the $|\funset|\cdot|\errset|$ different piecewise $\varepsilon$\hyp{}approximations, we only keep track of the $|\funset|\cdot|\errset|$ edges of the form $(i,j)$ that overlap the currently visited node $k$, i.e. $i \leq k < j$, and split them on-the-fly into prefix and suffix edges of the forms $(i,k)$ and $(k,j)$, respectively.

\looseness=-1
\Cref{alg:neat-partitioning} formalises this description. We use $\dist[k]$ to store an upper bound on the cost of the shortest path from node 1 to $k$, and $\prev[k]$ to store the previous node and corresponding fragment in the shortest path. We use $J_{f,\varepsilon}$ to keep track of the start/end positions of the fragment overlapping~$k$ and the parameters of the corresponding function of the $f$-kind that $\varepsilon$\hyp{}approximates it. We initialise and update $J_{f,\varepsilon}$ in Line~\ref{li:make-approx} with a call to $\textsc{MakeApproximation}(T,k,f,\varepsilon)$, which runs the algorithm of~\Cref{thm:nonlinear} starting from the data point~$T[k]$. Lines~\ref{li:start-prefix}--\ref{li:end-prefix} and~\ref{li:start-suffix}--\ref{li:end-suffix} relax prefix and suffix edges, respectively, and Lines~\ref{li:start-read}--\ref{li:end-read} conclude the algorithm by reading and returning the shortest path.

\begin{algorithm}[t]
  \caption{Partitioning a time series with \ourcompressor.}\label{alg:neat-partitioning}
  \begin{algorithmic}[1]
    \Require Time series $T[1,n]$, set $\funset$ of functions, set $\errset$ of error bounds
    \Ensure A partitioning of $T$ into fragments, each associated with an $\varepsilon$\hyp{}approximation $f$ (with $\varepsilon \in \errset$, $f \in \funset$), that minimises the size of the \ourcompressor encoding of $T$
    \State $\dist[1,n+1] \gets [\infty, \dots, \infty]$
    \State $\prev[1,n+1] \gets [\textsc{Null}, \dots, \textsc{Null}]$
    \ForAll{$(f,\varepsilon) \in \funset \times \errset$} \Comment{Initialise edges and functions}
        \State $J_{f, \varepsilon}.\mathit{start} \gets -\infty$
        \State $J_{f, \varepsilon}.\mathit{end} \gets -\infty$
        \State $J_{f, \varepsilon}.\mathit{params} \gets \textsc{Null}$
    \EndFor
    \For{$k \gets 1$ \textbf{to} $n$}
        \ForAll{$(f,\varepsilon) \in \funset \times \errset$}
            \If{$J_{f, \varepsilon}.\mathit{end} \leq k$} \Comment{A new edge overlaps $k$}
            \State $J_{f, \varepsilon} \gets \textsc{MakeApproximation}(T, k, f, \varepsilon)$\label{li:make-approx}
            \Else
                \State $i \gets J_{f, \varepsilon}.\mathit{start}$ \Comment{Relax prefix edge $(i,k)$}\label{li:start-prefix}
                \If{$\dist[k] > \dist[i] + w_{f,\varepsilon}(i,k)$}  
                    \State{$\dist[k] \gets \dist[i] + w_{f,\varepsilon}(i,k)$} 
                    \State{$\prev[k] \gets (i, J_{f, \varepsilon})$}\label{li:end-prefix}
                \EndIf
            \EndIf
        \EndFor
        \ForAll{$(f,\varepsilon) \in \funset \times \errset$} 
            \State $j \gets J_{f, \varepsilon}.\mathit{end}$\Comment{Relax suffix edge $(k,j)$}\label{li:start-suffix}
            \If{$\dist[j] > \dist[k] + w_{f,\varepsilon}(k,j)$} 
                \State{$\dist[j] \gets \dist[k] + w_{f,\varepsilon}(k,j)$} 
                \State{$\prev[j] \gets (k, J_{f, \varepsilon})$}\label{li:end-suffix}
            \EndIf
        \EndFor
    \EndFor
    \State $\mathit{result} \gets $ an empty dynamic array\label{li:start-read}
    \State $k \gets n+1$
    \While{$k \neq 1$} \Comment{Read the shortest path backwards}
        \State $\mathit{result}.\textsc{PushFront}(\prev[k])$
        \State $k \gets $ the first element of $\prev[k]$
    \EndWhile
    \State \Return $\mathit{result}$\label{li:end-read}
  \end{algorithmic}
\end{algorithm}

\paragraph*{Complexity analysis}
We now discuss the time complexity of \Cref{alg:neat-partitioning}.
For a fixed $f$ and $\varepsilon$, the overall contribution of Line~\ref{li:make-approx} to the time complexity is $O(n)$, since it eventually computes via~\Cref{thm:nonlinear} the piecewise $\varepsilon$\hyp{}approximation of $T$ composed of a function of the $f$-kind. Since there are $|\funset|\cdot|\errset|$ possible pairs of $f$ and $\varepsilon$, the overall computation of piecewise approximations takes $O(|\funset|\,|\errset|\, n)$~time.
It is easy to see that the relaxation of all the prefix and suffix edges runs within that same asymptotic time bound, thus the overall time complexity of~\Cref{alg:neat-partitioning} is $O(|\funset|\,|\errset|\, n)$.

Concerning $|\funset|$, we can assume that a real-world time series can be approximated well by a fixed number of function kinds, such as those in \Cref{tab:examples}, and thus it holds $|\funset|=O(1)$.
Concerning $|\errset|$, instead, we can pessimistically bound it as follows. Let $\Delta$ be one plus the difference between the maximum value and the minimum value $\hat{y}$ in~$T$. Then, each value $y_k$ of $T$ can be stored in $\lceil \log \Delta \rceil$~bits by just encoding the binary representation of $y_k-\hat{y}$.
This, in turn, entails that we can restrict our attention to the set $\errset = \{0,2^1, \dots, 2^{\lceil \log \Delta \rceil}\}$, since higher values of $\varepsilon$ would not pay off, i.e. even the most trivial constant function can $\varepsilon$\hyp{}approximate the whole time series.
Given that such a set $\errset$ has size $O(\log\Delta)$, the time complexity of~\Cref{alg:neat-partitioning} under these conditions is $O(n\log\Delta)$.

The average value of $\log \Delta$ for the diverse dataset we use in \Cref{sec:experiments} is 28.8, which is a small constant.
Moreover, we do not actually need to use all $\funset \times \errset$ pairs in \Cref{alg:neat-partitioning} but rather those surviving a model-selection procedure. For instance, we can initially run \Cref{alg:neat-partitioning} on a small sample of~$T$ (chosen e.g. according to the seasonality of~$T$) and select just the pairs that are used in the result, as these are likely to be effective and enough for the whole time series too. Our experiments show that this model-selection procedure improves the compression speed by an order of magnitude, with little impact on the compression ratio (see \Cref{sssec:vs-compression-speed}).

\paragraph*{Partitioning for lossy compression}
We can easily modify~\Cref{alg:neat-partitioning} to obtain a lossy representation of the time series $T$ with a given $\varepsilon$-bound on the error, still using functions from a given set~$\funset$ and minimising the space, which is given this time by just the storage of the functions' parameters (since we drop the corrections). 
It is enough to set $\errset=\{\varepsilon\}$ and define the edge weight $w_{f}(i,j)$ to be equal to the space in bits taken by the parameters of~$f$.
The resulting algorithm runs in $O(|\funset|\, n)$, so in linear time if $|\funset|=O(1)$.

Our experiments will show that, for a fixed $\varepsilon$-bound, this algorithm produces more succinct lossy representations of time series than known algorithms based on linear or nonlinear functions (namely, the algorithm by O'Rourke~\cite{ORourke:1981} and the Adaptive Approximation algorithm~\cite{Xu:2012,Qi:2015}, respectively).

\subsection{Designing the \ourcompressor compressor}\label{sec:random-access}

We now describe the layout of the compressed time series and how to support the random access operation.
As common in the literature~\cite{Chiarot:2022survey,Liakos:2022,Pelkonen:2015,Li:2023,ALP}, we focus on the storage of the values $y_1,\dots,y_n$ and assume the timestamps are $1,\dots,n$.\footnote{\label{foot:timestamps}The timestamps $x_1,\dots,x_n$ form an increasing sequence of integers that can be easily mapped to $1,\dots,n$ via monotone minimal perfect hash functions~\cite{Ferragina2023lemon} or compressed rank data structures~\cite{Boffa:2022talg,Ferragina2022repetitions}: the former are very succinct (about 3~bits per integer), the latter take more space but enable range queries over timestamps.\label{foot:timestamps}} 

Let us assume that the output of~\Cref{alg:neat-partitioning} is a sequence of $m$ tuples having the form $\langle\mathit{f}_i,\allowbreak\mathit{params}_i,\varepsilon_i,\mathit{start}_i,\mathit{end}_i\rangle$, where each tuple indicates a fragment $T[\mathit{start}_i,\mathit{end}_i]$ of $T[1,n]$ that is $\varepsilon_i$\hyp{}approximated by a function of kind $\mathit{f}_i \in \funset$ with parameters $\mathit{params}_i$.
We encode these $m$ tuples and the values in their corresponding time series fragments via:

\begin{itemize}[leftmargin=*]
    \item An integer array $\vstart[1,m]$ storing in $\vstart[i]$ the starting position of the $i$th fragment, i.e. $\vstart[i]=\mathit{start}_i$. To obtain the index of the fragment that covers a certain data point $T[k]$, we use the $\vstart.\rank(k)$ operation, which returns the number of elements in $\vstart$ that are smaller than or equal to $k$. Since $\vstart$ is an increasing integer sequence, we compress it via the Elias-Fano encoding~\cite{Elias:1974,Fano:1971}, which supports accessing an element in $O(1)$~time and $\vstart.\rank$ in $O(\min(\log m,\log\tfrac{n}{m}))$~time~\cite{Navarro:2016book}. 

    \item An integer array $\vepsilon[1,m]$ storing in $\vepsilon[i]$ the bit size of the corrections of the $i$th fragment, i.e. $\vepsilon[i]=\lceil \log (2\varepsilon_i+1)\rceil$.

    \item An integer array $\voffcorr[1,m+1]$ storing in $\voffcorr[i]$ the cumulative bit size of the corrections in the fragments preceding the $i$th one, i.e. $\voffcorr[i]=\sum_{j=1}^{i-1} \vepsilon[j] \, (\mathit{end}_j-\mathit{start}_j+1)$. Notice that $O[m+1]$ denotes the overall bit size of the corrections. Similarly to $\vstart$, we compress $\voffcorr$ via the Elias-Fano encoding. 

    \item A bit string $\vcorrections[1,\voffcorr[m+1]]$ storing in $\vcorrections[\voffcorr[i],\voffcorr[i+1]-1]$ the correction values $y_j-\lfloor f_i(x_j) \rfloor$ of the $i$th fragment, where $j \in [\mathit{start}_i, \mathit{end}_i]$.
    
    \item An integer array $\vkinds[1,m]$ storing in $\vkinds[i]$ the function kind for the $i$th fragment, i.e. $\vkinds[i]=\mathit{f}_i$. We regard $\vkinds$ as a string over the alphabet $\{1, \dots, |\funset|\}$ and represent it as a wavelet tree data structure~\cite{GrossiGV03wavelet,Navarro:2016book}. This allows us to compute the $\vkinds.\rank_f(i)$ operation, which returns the number of occurrences of the function kind $f$ in $\vkinds[1,i]$ in $O(\log |\funset|)=O(1)$~time.

    \item For each $f \in \funset$, an array $\vparams_f$ concatenating the parameters $\mathit{params}_i$ of the functions of the same kind $f$. This way, the parameters $\mathit{params}_i$ of the $i$th fragment can be found in $\vparams_{f_i}[\vkinds.\rank_{f_i}(i)]$.
\end{itemize}

All the above arrays use cells whose bit size is just enough to contain the largest value stored in them. If $\funset$ contains functions with the same number of parameters (recall from \Cref{ssec:nonlinear} that we can use functions with more than two parameters), we can simplify the above encoding by avoiding the use of a wavelet tree for $\vkinds$ and by concatenating all the functions' parameters $\vparams_f$ into a single array, which is accessed simply through the index of the queried fragment.

Having defined how we represent $T$ in compressed form via a tuple $\langle\vstart,\vepsilon,\voffcorr,\vcorrections,\vkinds,\vparams\rangle$ of data structures, we are now ready to discuss the decompression and random access operations.

\Cref{alg:neat-decompression} shows how to decompress the whole time series.
For each fragment, we first decode the associated boundaries and kind of approximation function (Lines~\ref{li:decompress-1}--\ref{li:decompress-2}), and then we output all the values $y_k$ within the fragment's boundaries by applying the function to index $k$ and adding the corresponding correction value (Lines~\ref{li:decompress-3}--\ref{li:decompress-4}). It is easy to see that the time complexity of~\Cref{alg:neat-decompression} is $O(n)$ given that all the involved operations take constant time.
Furthermore, since each data point is decompressed independently from the others, the algorithm could be parallelised trivially by decompressing different fragments with different workers, and the computation of the function within a fragment could be implemented via SIMD instructions.

\begin{algorithm}[t]
  \caption{Full decompression in \ourcompressor.}\label{alg:neat-decompression}
  \begin{algorithmic}[1]
    \Require The \ourcompressor encoding $\langle\vstart,\vepsilon,\voffcorr,\vcorrections,\vkinds,\vparams\rangle$ of $T$
    \Ensure The uncompressed values of $T$ 
    \State $o \gets 1$ \Comment{Bit-offset to the correction} 
    \For{$i \gets 1$ \textbf{to} $m$} \Comment{For each fragment}\label{li:decompress-0}
        \State $\mathit{start} \gets \vstart[i]$ \Comment{First data point index} \label{li:decompress-1}
        \State $\mathit{end} \gets \vstart[i+1]-1$ \Comment{Last data point index}
        \State $f \gets \vkinds[i]$ \Comment{Function kind}
        \State $\mathit{params} \gets \vparams_{f}[\vkinds.\rank_{f}(i)]$ \Comment{Function parameters}
        \State $b \gets \vepsilon[i]$ \Comment{Correction bit size} \label{li:decompress-2}
        \For{$k \gets \textit{start}$ \textbf{to} $\textit{end}$}\label{li:decompress-3}
            \State $\tilde{y} \gets $ compute $\lfloor f(k)\rfloor$ using $\mathit{params}$
            \State \textbf{output} $\tilde{y} + \mathit{int}(C[o,o+b-1])$
            \State $o \gets o + b$ \label{li:decompress-4}
        \EndFor
    \EndFor
  \end{algorithmic}
\end{algorithm}
\begin{algorithm}[t]
  \caption{Random access in \ourcompressor.}\label{alg:neat-ra}
  \begin{algorithmic}[1]
    \Require An index $k$, the \ourcompressor encoding $\langle\vstart,\vepsilon,\voffcorr,\vcorrections,\vkinds,\vparams\rangle$ of $T$
    \Ensure The value of $T[k]$
    \State $i \gets \vstart.\rank(k)$ \Comment{Index of the fragment}\label{li:ra-1}
    \State $\mathit{start} \gets \vstart[i]$ \Comment{First data point index}\label{li:ra-2}
    \State $f \gets \vkinds[i]$ \Comment{Function kind}
    \State $\mathit{params} \gets \vparams_{f}[\vkinds.\rank_{f}(i)]$ \Comment{Function parameters}
    \State $b \gets \vepsilon[i]$ \Comment{Correction bit size}\label{li:ra-3}
    \State $\tilde{y} \gets $ compute $\lfloor f(k)\rfloor$ using $\mathit{params}$\label{li:ra-4}
    \State $o \gets \voffcorr[i]+(k-\mathit{start})b$ \Comment{Bit-offset to the correction}
    \State \Return $\tilde{y} + \mathit{int}(\vcorrections[o,o+b-1])$\label{li:ra-5}
  \end{algorithmic}
\end{algorithm}

\Cref{alg:neat-ra} shows how to perform the random access operation to the value of $T[k]$, for a given index $k \in \{1, 2, \ldots, n\}$. We start by identifying the index of the fragment where $T[k]$ falls into (Line~\ref{li:ra-1}), then we decode the index of the first data point in that fragment and the (kind and parameters of) function associated with that fragment  (Lines~\ref{li:ra-2}--\ref{li:ra-3}), and finally we apply the function to position~$k$ and add the corresponding correction value (Lines~\ref{li:ra-4}--\ref{li:ra-5}). The time complexity of~\Cref{alg:neat-ra} is dominated by the operation $\vstart.\rank$ at Line~\ref{li:ra-1}, which takes $O(\min(\log m,\log\tfrac{n}{m}))$~time.
We can easily achieve $O(1)$~time by representing $\vstart$ as a bitvector of length~$n$ with a~$\mathtt{1}$ in each position $\mathit{start}_i$, and then using the well-known constant-time \rank/\select operations~\cite{Jacobson:1989,Clark:1996}.

\section{Experiments}\label{sec:experiments}

\subsection{Experimental setting}

We run our experiments on a machine with 1.17~TiB of RAM and an Intel Xeon Gold 6140M CPU, running CentOS~7.
Our code is in C\texttt{++}23, compiled with GCC 13.2.1, and publicly available at
{\url{https://github.com/and-gue/NeaTS}}.
We refer to our lossy and lossless approaches as \ourcompressorlossy and \ourcompressor, respectively. 
We use four types of functions\,---\,namely, linear, exponential, quadratic, and radical\,---\,which turned out to be sufficient to capture the trends in our real-world datasets well.
We use vector instructions in our decompression procedures via the \texttt{std::experimental::simd}~library, and succinct data structures from the \texttt{sdsl}~\cite{gbmp2014} and \texttt{sux}~\cite{Vigna:2008} libraries.

\subsubsection{Datasets}
We use 16 real-world time series datasets out of which 13 were sourced by Chimp~\cite{Liakos:2022}, and the remaining 3 were obtained from the Geolife project~\cite{zheng2011geolife} and a study on arrhythmia~\cite{ecg,ecgpaper}. 
{ Consistent with previous studies~\cite{Liakos:2022,Pelkonen:2015,Li:2023,ALP}, we ignore  timestamps as they are either consecutive increasing integers or can be transformed into such with other ad hoc data structures (see \Cref{foot:timestamps}).}
All the datasets report values in textual fixed-precision format, therefore, unless the compressor is designed for doubles, we transform them into 64-bit integers by multiplying each value by a factor~$10^{x}$, where $x$ is the number of fractional digits.

\begin{itemize}[leftmargin=*]
    \item[-] \textbf{IR-bio-temp (IT)}~\cite{IRBioTemp} contains about 477M biological temperature observations from an infrared sensor, with 2 fractional digits.
    
    \item[-] \textbf{Stocks-USA (US)}, \textbf{Stocks-UK (UK)}, and \textbf{Stocks-DE (GE)}~\cite{Stocks}  contain about 282M, 59M, and 43M stock exchange prices of USA, UK, and Germany, with 2, 1, and 3 fractional digits, respectively.
    
    \item[-] \textbf{Electrocardiogram (ECG)}~\cite{ecg,ecgpaper} contains about 226M electrocardiogram signals of over 45K patients, with 3 fractional digits.
    
    \item[-] \textbf{Wind-direction (WD)}~\cite{WindDir} contains about 199M wind direction observations, with 2 fractional digits.
    
    \item[-] \textbf{Air-pressure (AP)}~\cite{AirPressure} contains about 138M timestamped values of barometric pressure corrected to sea level and surface level, with 5 fractional digits.
    
    \item[-] \textbf{Geolife-longitude (LON)}, and \textbf{Geolife-latitude (LAT)}~\cite{zheng2011geolife} contain about 25M timestamped longitude and latitude values of 182 users' GPS trajectories, with 4 fractional digits.
    
    \item[-] \textbf{Dewpoint-temp (DP)}~\cite{DewpointTemp} contains about 5M relative dew point temperature observations, with 3 fractional digits.

    \item[-] \textbf{City-temp (CT)}~\cite{CityTemp} contains about 3M temperature observations of cities around the world, with 1 fractional digit.
    
    \item[-] \textbf{PM10-dust (DU)}~\cite{PM10Dust} contains about 334K measurements of PM10 in the atmosphere, with 3 fractional digits.
    
    \item[-] \textbf{Basel-wind (BW)}, and \textbf{Basel-temp (BT)}~\cite{Basel} contain about 130K records of wind speed and temperature data of Basel (Switzerland), with 7 and 9 fractional digits, respectively.
    
    \item[-] \textbf{Bird-migration (BM)}~\cite{influxdb2data} contains about 18K positions of birds, with 5 fractional digits.
    
    \item[-] \textbf{Bitcoin-price (BP)}~\cite{influxdb2data} contains about 7K prices of Bitcoin in the dollar exchange rate, with 4 fractional digits.
\end{itemize}

\subsubsection{Competitors}\label{sssec:competitors}

Regarding the lossy compressors, we compare our \ourcompressorlossy against 2 functional approximation algorithms: the optimal Piecewise Linear Approximation algorithm (PLA)~\cite{ORourke:1981}, and the Adaptive Approximation algorithm (AA)~\cite{Xu:2012,Qi:2015} that combines linear, exponential and quadratic functions. We implemented the PLA and the AA algorithms in C\texttt{++} since their code is not publicly available.

Regarding the lossless compressors, we compare our \ourcompressor against 5 widely-used general-purpose compressors\,---\,namely, \xz~\cite{Xz}, \brotli~\cite{Alakuijala:2018}, \zstd~\cite{Collet:2016}, \lz~\cite{Collet:2013}, and \snappy~\cite{Snappy}\,---\,and {7} state-of-the-art special-purpose compressors\,---\,namely, \chimp and \chimpo~\cite{Liakos:2022}, \tsxor~\cite{Bruno:2021}, \DAC~\cite{Brisaboa:2013}, \gorilla~\cite{Pelkonen:2015}, \leco~\cite{leco:2024}, and \alp~\cite{ALP}.
We use the Squash library ~\cite{squashlib} for all the general-purpose compressors.
We use the public implementations of \tsxor and \gorilla available in the repository of~\cite{Bruno:2021}, the implementation of \DAC available in \texttt{sdsl}~\cite{gbmp2014}, and the original implementations of \leco\ and \alp.
We ported \chimp and \chimpo to C\texttt{++} since their original implementations are in Java.

\looseness=-1
Following~\cite{Liakos:2022,Li:2023}, we apply compressors that do not natively support random access (thus excluding \DAC, \leco, and  \ourcompressor) to blocks of 1000 consecutive values. We then maintain an array that maps each block index to a pointer referencing the starting byte of the block in the compressed output.

\subsection{On the lossy compressors}\label{ssec:exp-lossy}

To evaluate the lossy compressors with a meaningful error-bound parameter~$\varepsilon$, we determined the smallest $\varepsilon$ such that \ourcompressorlossy achieves better compression than our lossless compressor \ourcompressor.
We express the resulting $\varepsilon$ as a \% of the range of values (i.e. largest minus smallest value) in a dataset, and the compression ratio as the size of the compressed output divided by the size of the original data.%

The results in~\Cref{tab:lossy} show that \ourcompressorlossy outperforms in compression ratio both the PLA and the AA algorithms on all datasets.
On average, \ourcompressorlossy improves the compression ratio of PLA by 7.02\% and the one of AA by 11.77\%.
This demonstrates that, under the same $\varepsilon$-bound, the use of nonlinear functions allows achieving better compression compared to linear functions alone (as in the widely-used PLA).
In turn, despite employing nonlinear approximations, AA is worse than PLA for nearly all datasets due to the use of a heuristic technique to partition the time series into fragments and of a sub-optimal algorithm for (non)linear $\varepsilon$\hyp{}approximations, two issues that our \ourcompressorlossy solve.

{ For the approximation accuracy, we report that the Mean Absolute Percentage Error (MAPE)\,---\,i.e. the mean of the absolute relative errors between the approximated and the actual values, expressed as a percentage\,---\,is 2.47\% for AA, 2.85\% for \ourcompressorlossy, and 4.37\% for PLA (on average over all datasets).
Therefore, \ourcompressorlossy has a much better accuracy than PLA and a slightly worse accuracy than AA.
This is because AA creates more time series fragments than \ourcompressorlossy, and its functions pass through the first data point of each fragment: two factors that together yield zero errors on many data points.

In terms of compression speed, PLA is the fastest at 123.36~MB/s, followed by AA at 63.11~MB/s, and \ourcompressorlossy at 18.23~MB/s.
These results reflect the higher computational effort of \ourcompressorlossy in achieving better compression ratios.

In terms of decompression speed, PLA is the fastest at 2997.00~MB/s, followed by \ourcompressor at 2561.31~MB/s, and AA at 2420.20~MB/s. This can be attributed to the fact that the linear models in PLA are faster to evaluate, and that \ourcompressor uses fewer fragments than AA, thus reducing the overhead associated with switching between fragments.}

\begin{table}[t]
\caption{Compression ratios of the 3 experimented lossy approaches\,---\,i.e. AA, PLA, and \ourcompressorlossy{}\,---\,on the 16 datasets.}%
\renewcommand{\arraystretch}{0.97}
\setlength{\tabcolsep}{5pt}
\begin{tabular}[b]{clrrrrr}
\toprule
\multirow{2}{*}{Dataset} & \multirow{2}{*}{$\varepsilon$ (\%)} & \multicolumn{3}{c}{Compression ratio (\%)} & \multicolumn{2}{c}{\ourcompressorlossy improv. (\%)} \\  \cmidrule(lr){3-5} \cmidrule(lr){6-7}
&  & AA   & PLA & \ourcompressorlossy & wrt AA & wrt PLA \\

\midrule
IT  & 1.15E-1 & 12.11 & 12.07 & \textbf{11.07} & 8.57 & 8.29 \\
US  & 2.40E-3 & 7.96 & 7.41 & \textbf{6.99 }& 12.09 & 5.65 \\
 ECG & 5.43E-2 & 15.03 & 13.46 & \textbf{12.97} & 13.71 & 3.64 \\
WD & 6.36E-0 & 28.09 & 26.94 & \textbf{24.76} & 11.88 & 8.11 \\
AP & 3.08E-3 & 21.90 & 20.00 & \textbf{19.17} & 12.49 & 4.16 \\
UK & 9.53E-3 & 9.82 & 9.21 & \textbf{8.69 }& 11.50 & 5.63 \\
GE &  9.12E{-3} & 13.95 & 12.79 & \textbf{12.08} & 13.35 & 5.52 \\
LAT & 7.00E{-6} & 25.40 & 23.59 & \textbf{22.09} & 13.03 & 6.35 \\
LON & 1.40E{-5} & 19.92 & 18.32 & \textbf{17.26} & 13.37 & 5.78 \\
DP & 6.32E{-2} & 17.51 & 16.89 & \textbf{15.87} & 9.35 & 6.07 \\
CT & 3.88E{0} & 16.19 & 14.45 & \textbf{13.92} & 14.03 & 3.69 \\
DU & 6.00E{-3} & 10.04 & 10.32 & \textbf{9.15 }& 8.93 & 11.39 \\
BT & 4.85E{-1} & 59.62 & 61.29 & \textbf{53.77} & 9.81 & 12.26 \\
BW & 3.16E{-3} & 52.19 & 48.28 & \textbf{45.01} & 13.75 & 6.77 \\
BM & 1.42E{-2} & 27.13 & 25.32 & \textbf{23.29} & 14.15 & 8.00 \\
BP & 3.61E{-1} & 43.05 & 41.76 & \textbf{38.52} & 10.54 & 7.76 \\

\bottomrule
\end{tabular}%
\label{tab:lossy}
\end{table}

\subsection{On the lossless compressors}
We now compare our \ourcompressor against the 5~lossless general-purpose compressors (i.e. \xz, \brotli, \zstd, \lz, and \snappy) and the {7}~lossless special-purpose compressors (i.e. \chimp, \chimpo, \tsxor, \DAC, \gorilla, \leco, and \alp).

\Cref{tab:lossless} reports the compression ratio, decompression speed, and random access speed of all compressors on each dataset, where the best result in each family of compressors is in bold, and the best result overall is underlined.
Moreover, we plot the trade-offs compression ratio vs compression speed, compression ratio vs decompression speed, and compression ratio vs random access speed in \Cref{fig:compression,fig:decompression,fig:random_access} and dig into them in \Cref{sssec:vs-compression-speed,sssec:vs-decompression-speed,sssec:vs-random-access-speed}, respectively.
{ Then, in \Cref{sssec:range-queries}, we explore the benefits of our approach from a data management perspective by focusing on the important case of range queries.}
Finally, we provide a summary of our experiments in \Cref{sssec:summary}.

\subsubsection{Compression ratio vs compression speed}\label{sssec:vs-compression-speed}
\Cref{fig:compression} shows the trade-off between compression speed and compression ratio of the lossless compressors.

First, we notice that \xz, followed by \brotli, achieve on average (but not always, see below) the best compression ratio at the cost of a slow compression speed, indeed, they are at the bottom-left of \Cref{fig:compression}.  
The opposite extreme in this trade-off is occupied by the special-purpose compressor \gorilla (at the top-right of \Cref{fig:compression}), which is 3 orders of magnitude faster in compression speed than \brotli but achieves a compression ratio above 70\% on average.
In between these two extremes, we notice that \alp is on the Pareto front of this trade-off, dominating (in order of increasingly higher compression speeds and worse compression ratios) \leco, \tsxor, \DAC, \zstd, \chimp, \chimpo, \lz and \snappy, but not our \ourcompressor.

\begin{figure}[t]
\centering
\includegraphics[width=\linewidth]{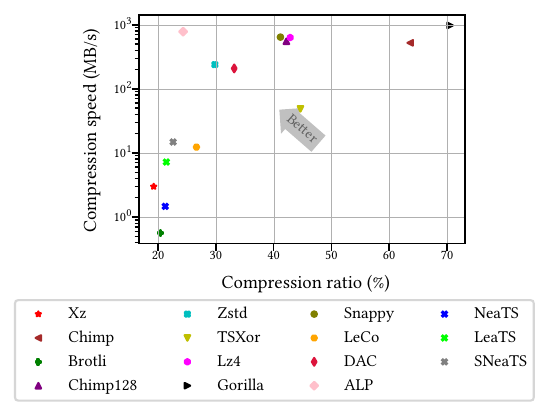}
\caption{The trade-off between compression ratio and speed of the lossless compressors, averaged on the 16 datasets.}
\label{fig:compression}
\end{figure}

\begin{figure*}[t]
\centering
\includegraphics[width=\textwidth]{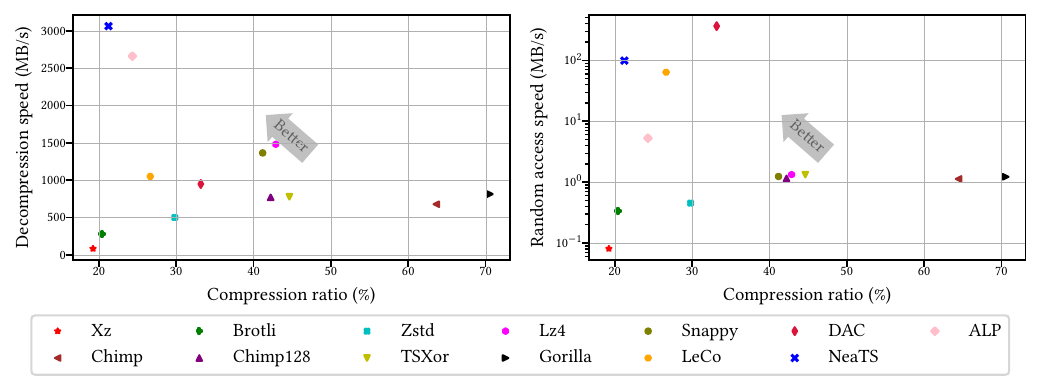}
\caption{The trade-off between compression ratio and decompression speed (left plot), and between compression ratio  and random access speed (right plot) of the lossless compressors, averaged on the 16 datasets. Note that the vertical axis of the right plot is logarithmic.}
\label{fig:random_access}
\label{fig:decompression}
\end{figure*}

Indeed, we observe from \Cref{tab:lossless} that \ourcompressor achieves the best compression ratio among the special-purpose compressors on 14/16 datasets, and the best compression ratio overall on 4/16 datasets.
Its compression speed is low but still 164.14\% faster than \brotli with just a 4.09\% worse compression ratio on average.
Moreover, \ourcompressor always achieves better compression ratios than \lz and \snappy (the fastest general-purpose compressors in terms of compression speed) by 52.77\% and 50.20\% on average, respectively.
\ourcompressor also achieves better compression ratios than \zstd for almost all the datasets (except for AP), with an average improvement of~{28.49\%}.

Compared to the special-purpose compressors, the only 2/{16} datasets in which \ourcompressor does not achieve the best compression ratio are BP and BT, where \alp is slightly better. However, these two datasets are also among the smallest ones, and \ourcompressor achieves a better compression ratio than \alp by 16.36\% on average.

\looseness=-1
We conclude this section by experimenting with two variants of \ourcompressor that improve the compression speed at the cost of worse compression. The first, named LeaTS, reduces the set of functions considered by \Cref{alg:neat-partitioning} to linear functions only. The second, named S\ourcompressor, reduces the set of functions and error bounds considered by \Cref{alg:neat-partitioning} to those surviving a model-selection procedure (included in the construction time) that picks the top-5 most-used pairs in the first 10\% of the dataset. The results, depicted in \Cref{fig:compression}, show that LeaTS and SNeaTS achieve a compression speed that is 5.22$\times$ and 12.86$\times$ that of \ourcompressor, and a compression ratio that is 0.89\% and 8.18\% worse than \ourcompressor, respectively.\footnote{In the lossless scenario, the space is clearly dominated by the storage of the corrections rather than the function parameters, which is why the improvement in compression ratio of nonlinear functions over linear ones is not as high as 12\%, as experienced in the lossy scenario (\Cref{tab:lossy}).} The latter variant, in particular, is both faster in compression speed than \leco by 36.26\% and better in compression ratio by 12.80\% on~average.

Despite their better compression ratios, these variants are still not as fast as \alp or \gorilla in compression speed. However, we anticipate from the next subsections that \ourcompressor also excels in decompression and random access speed, thus making it the most competitive compressor in a query-intensive and space-constrained scenario.
Moreover, if compression speed is key for the underlying application, we could imagine using a lightweight compressor like \alp or \gorilla when the time series is first ingested, and running \ourcompressor later on (or in the background) to provide much more effective compression and efficient query operations in the long~run.

\subsubsection{Compression ratio vs decompression speed}\label{sssec:vs-decompression-speed}
In an analytical scenario, the decompression speed is a crucial performance metric. 
The middle of~\Cref{tab:lossless} shows the decompression speed of all the compressors on all the datasets, while \Cref{fig:decompression} shows the trade-off between compression ratio and decompression speed averaged on all the datasets.

First, we notice from~\Cref{tab:lossless} that \ourcompressor achieves the fastest decompression speed on 10/16 datasets { thanks to its cache-friendly and vectorised decompression procedure}.
Compared to \alp, which obtains better performance on the remaining 6/16 datasets, \ourcompressor is both 16.36\% better in compression ratio and 27.08\% faster in decompression speed on average.
If we consider instead \xz and \brotli (i.e. the closest competitors to \ourcompressor in terms of compression ratio, as commented above), their decompression speeds are {44.92$\times$} and {12.27$\times$} lower on average than that of \ourcompressor, respectively.

\looseness=-1
Finally, \ourcompressor dominates all the other special-purpose compressors in this compression ratio vs decompression speed trade-off, as \Cref{fig:decompression} clearly shows by placing them at the bottom-right of \ourcompressor in the decompression plot.
For instance, compared to \leco, \ourcompressor is {18.23\%} better in compression ratio and {201.11\%}  
faster in decompression speed.

\subsubsection{Compression ratio vs random access speed}\label{sssec:vs-random-access-speed}
Another key performance metric for the efficient analysis of time series is the random access speed.
The bottom of~\Cref{tab:lossless} shows the average random access speed (for 10M queries) of all the compressors on all the datasets, while \Cref{fig:random_access} shows the trade-off between compression ratio and random access speed averaged on all the datasets.

\looseness=-1
First, we notice from \Cref{tab:lossless} that \DAC, followed by \ourcompressor, achieves the best random access speed.
However, we remark that \ourcompressor is much more effective than \DAC in terms of compression ratio, i.e. {37.25\%} better on average and up to {67.86\%} better overall.
This is why both \ourcompressor and \DAC occupy a prominent position in the compression ratio vs random access speed trade-off, as \Cref{fig:random_access} clearly shows by placing them close to the top-left edge of the random access plot.

\looseness=-1
\leco is the only other compressor supporting random access natively (i.e. without the block-wise approach described in \Cref{sssec:competitors}). \ourcompressor is both {118.55\%} faster in random access speed and {18.23\%} better in compression ratio than \leco.

The remaining compressors are from 2 to 3 orders of magnitude slower in random access speed than \ourcompressor. In particular, \ourcompressor dominates all other special-purpose compressors and the vast majority of general-purpose compressors in the compression ratio vs random access speed trade-off.
The only exceptions are \xz and \brotli that, compared to \ourcompressor, can provide slightly better compression ratios (except for the 4/{16} dataset where \ourcompressor is better) but much slower random access speeds (always), which is why they are at the bottom-left of every plot.

{
\subsubsection{Range queries}\label{sssec:range-queries}
The most fundamental queries in time series databases\,---\,such as trend analysis, anomaly detection, correlation analysis, and data aggregation\,---\,ultimately rely on accessing data within a specific time interval (i.e. a range query)~\cite{KhelifatiKDDC23,hao2021ts}, which boils down to a random access operation (to retrieve the first data point) followed by a scan (to retrieve the subsequent data points within the interval).
We thus now focus on the best compressors in terms of random access or decompression speed (i.e. \alp, \DAC, \lz, and \ourcompressor) and evaluate their range query performance at different range sizes, from $10 \cdot 2^0$ to $10 \cdot 2^{16}$ data points. 
\Cref{fig:range_queries} shows the throughput in queries per second measured on 10K random range queries and averaged over the 11 largest datasets.
For range sizes smaller than 40, \DAC is the fastest solution, followed by \ourcompressor, which remains an order of magnitude more efficient than other compressors. For larger range sizes, \ourcompressor clearly outperforms all competitors.
This demonstrates the ability of \ourcompressor to provide a full spectrum of efficient data access, from small to large ranges, thus benefiting a wide variety of queries in time series databases.
}

\begin{figure}[t]
\centering
\includegraphics[width=\linewidth]{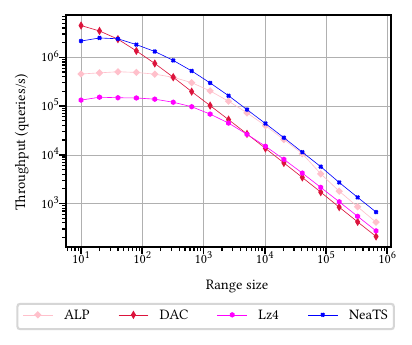}
\caption{ Range queries throughput across different range sizes.}
\label{fig:range_queries}
\end{figure}

\subsubsection{Summary}\label{sssec:summary}

Our results on 16 real-world time series show that \ourcompressor emerges as the only approach to date providing, simultaneously, compression ratios close to or better than the existing compressors (i.e. the best compression ratio among the special-purpose compressor on 14/16 datasets, and the best overall on 4/{16} dataset), a much faster decompression speed, and up to 3 orders of magnitude more efficient random access.

\looseness=-1
No other compressor can strike such a good trade-off among all these factors together.
For example, \xz achieves the best compression ratio on {9/16} datasets, but on average its decompression and random access speeds are {44.92$\times$} and {1657.10$\times$} slower than that of \ourcompressor, respectively.
\alp achieves the fastest compression speed (ignoring \gorilla, whose compression ratio is not very competitive), but on average \ourcompressor achieves 16.36\% better compression ratio, 27.08\% faster decompression speed, and at least one order of magnitude faster random access speed than \alp.
\DAC achieves the fastest random access speed, but on average \ourcompressor is 37.25\% better in compression ratio and 234.89\% faster in decompression speed than \DAC. 
{ Furthermore, \ourcompressor outperforms the other compressors for range queries involving 40 or more data points.}

This is evidence that \ourcompressor has the potential to be the compressor of choice for the storage and real-time analysis of massive and ever-growing amounts of time series data.

\begin{table*}[pt]
\caption{Compression ratio (top), decompression speed (middle), and random access speed (bottom) achieved by the 5 general-purpose and the {8} special-purpose lossless compressors (including our \ourcompressor) on 16 datasets, sorted by decreasing size. We highlight in bold the best result in each family, and in underline the best result overall.}
\label{tab:lossless}
\renewcommand{\arraystretch}{0.95}
\setlength{\tabcolsep}{5pt}
\centering
\begin{tabular}{c c  rrrrr  rrrrrrr r}
\toprule
\multicolumn{2}{c}{\multirow{2}{*}{Dataset}} &
  \multicolumn{5}{c}{General-purpose compressors} &
  \multicolumn{7}{c}{Special-purpose compressors} \\  %
  \cmidrule(lr){3-7} \cmidrule(l){8-15}
 & &
  \xz &
  \brotli &
  \zstd &
  \lz &
  \snappy &
  \chimpo &
  \chimp &
  \tsxor &
  \DAC &
  \gorilla &
  \leco &
  \alp &
  \ourcompressor \rule{0pt}{2.6ex} \\ 
\midrule
\multirow{15}{*}{\rotatebox{90}{Compression ratio (\%)}}
& IT & \textbf{12.86} & 14.25 & 23.46 & 41.31 & 36.96 & 29.43 & 72.30 & 30.76 & 23.83 & 78.60 & 13.62 & 16.86 & \textbf{\underline{11.88}} \\
& US & 9.18 & \textbf{8.70 }& 12.82 & 27.09 & 21.51 & 18.94 & 54.55 & 18.89 & 24.95 & 57.54 & 9.16 & 10.50 & \textbf{\underline{8.02}}\\
&  ECG & \textbf{\underline{12.12}} & \textbf{\underline{12.12}} & 17.04 & 26.14 & 33.75 & 54.11 & 43.18 & 20.03 & 25.39 & 45.26 & 15.58 & 16.23 & \textbf{12.96} \\
& WD & \textbf{\underline{23.60}} & 27.60 & 33.78 & 52.70 & 54.19 & 43.38 & 84.09 & 46.42 & 25.75 & 91.02 & 24.71 & 24.90 & \textbf{24.37} \\
& AP & \textbf{\underline{12.35}} & 12.69 & 17.87 & 26.50 & 24.82 & 30.00 & 35.76 & 34.78 & 41.13 & 37.67 & 23.52 & 25.74 & \textbf{19.27} \\
& UK & 9.42 & \textbf{\underline{9.06 }}& 12.99 & 26.94 & 21.41 & 23.13 & 46.95 & 15.85 & 25.79 & 53.92 & 10.83 & 11.64 & \textbf{9.09 }\\
& GE & 11.07 & \textbf{\underline{11.04}} & 15.27 & 30.25 & 23.94 & 21.08 & 66.90 & 21.44 & 29.01 & 71.49 & 13.43 & 13.88 & \textbf{12.11} \\
& LON & \textbf{\underline{17.03}} & 18.63 & 32.72 & 49.71 & 49.28 & 58.64 & 61.70 & 71.64 & 47.27 & 63.09 & 20.74 & 26.87 & \textbf{17.53} \\
& LAT & \textbf{\underline{21.51}} & 23.67 & 40.77 & 52.12 & 51.44 & 58.09 & 61.44 & 71.93 & 47.27 & 65.02 & 25.56 & 26.70 & \textbf{22.22} \\
& DP & \textbf{16.37} & 17.02 & 29.35 & 48.61 & 47.54 & 49.53 & 77.17 & 60.91 & 26.95 & 83.53 & 17.83 & 22.04 & \textbf{\underline{16.10}} \\
& CT & \textbf{15.72} & 16.37 & 25.33 & 42.92 & 37.31 & 36.09 & 73.25 & 30.96 & 19.14 & 87.11 & 17.91 & 15.27 & \textbf{\underline{14.20}} \\
& DU & 8.21 & \textbf{\underline{7.78 }}& 11.37 & 23.00 & 18.62 & 21.68 & 39.74 & 18.31 & 11.14 & 44.49 & 28.54 & 13.34 & \textbf{9.46 }\\
& BT & \textbf{\underline{45.66}} & 45.69 & 58.12 & 67.20 & 68.64 & 46.90 & 84.01 & 53.88 & 57.07 & 92.88 & 58.15 & \textbf{46.25} & 54.01 \\
& BW & \textbf{\underline{36.17}} & 41.49 & 50.24 & 58.74 & 58.79 & 71.27 & 87.16 & 82.32 & 45.91 & 99.72 & 56.99 & 50.01 & \textbf{45.21} \\
& BM & \textbf{\underline{19.67}} & 20.70 & 29.52 & 43.58 & 39.39 & 40.96 & 61.96 & 48.18 & 37.42 & 74.67 & 50.72 & 30.80 & \textbf{23.44} \\
& BP & \textbf{\underline{36.97}} & 39.85 & 66.43 & 69.03 & 71.22 & 72.09 & 67.84 & 87.86 & 42.79 & 82.72 & 39.03 & \textbf{38.37} & 39.89 \\
\cmidrule(lr){1-7} \cmidrule(lr){8-15}
\multirow{15}{*}{\rotatebox{90}{Decompression speed (MB/s)}}
& IT & 90.83 & 304.65 & 459.91 & \textbf{1405.36} & 1207.61 & 725.74 & 598.36 & 743.69 & 999.45 & 795.22 & 1082.45 & 2249.26 & \textbf{\underline{2549.04}} \\
& US & 133.37 & 396.11 & 643.89 & 1609.06 & \textbf{1928.74} & 1109.06 & 692.86 & 1084.44 & 896.80 & 839.86 & 1097.74 & 2295.05 & \textbf{\underline{2982.09}} \\
&  ECG & 94.39 & 253.72 & 512.45 & 1325.59 & \textbf{1473.75} & 559.59 & 645.71 & 773.99 & 1082.36 & 790.96 & 1306.75 & 2344.27 & \textbf{\underline{2897.08}} \\
& WD & 56.94 & 227.86 & 490.46 & 1443.34 & \textbf{1191.21} & 732.56 & 606.28 & 754.89 & 864.86 & 854.76 & 1035.08 & \textbf{\underline{2253.39}} & 1936.42 \\
& AP & 106.07 & 332.69 & 683.65 & \textbf{1740.56} & 1599.33 & 885.36 & 893.27 & 885.33 & 705.69 & 978.22 & 1013.16 & 2116.29 & \textbf{\underline{2944.05}} \\
& UK & 131.95 & 392.03 & 634.25 & 1645.54 & \textbf{1815.86} & 853.25 & 670.05 & 1102.54 & 829.75 & 863.34 & 1087.50 & 2312.32 & \textbf{\underline{3015.90}} \\
& GE & 115.76 & 350.47 & 594.39 & \textbf{1611.32} & 1761.48 & 1008.27 & 656.04 & 986.84 & 962.52 & 829.82 & 1062.78 & 2307.16 & \textbf{\underline{3243.73}} \\
& LAT & 53.10 & 171.93 & 376.31 & \textbf{1289.07} & 1033.77 & 758.29 & 785.61 & 492.87 & 1019.50 & 622.10 & 960.91 & 2144.64 & \textbf{\underline{2935.03}} \\
& LON & 61.55 & 212.49 & 375.82 & \textbf{1198.19} & 963.81 & 776.44 & 801.64 & 493.81 & 1025.93 & 625.60 & 925.42 & 2113.61 & \textbf{\underline{2870.15}} \\
& DP & 75.14 & 293.10 & 429.39 & \textbf{1456.28} & 1098.91 & 581.50 & 626.82 & 618.88 & 1056.31 & 782.95 & 963.67 & \textbf{\underline{2057.49}} & 1953.35 \\
& CT & 79.93 & 298.52 & 458.10 & \textbf{1467.29} & 1225.00 & 760.50 & 556.01 & 768.86 & 1070.76 & 841.18 & 969.50 & \textbf{\underline{3290.46}} & 3269.65 \\
& DU & 147.58 & 436.39 & 696.34 & 1701.39 & \textbf{2031.31} & 923.52 & 805.54 & 1093.90 & 658.38 & 948.68 & 746.33 & \textbf{\underline{5410.27}} & 4007.91 \\
& BT & 33.43 & 173.70 & 414.01 & \textbf{1418.67} & 1039.45 & 662.70 & 590.54 & 672.81 & 783.96 & 805.45 & 846.57 & \textbf{\underline{4241.33}} & 2570.33 \\
& BW & 41.74 & 153.77 & 378.60 & \textbf{1527.78} & 1020.14 & 607.71 & 611.78 & 655.01 & 1189.51 & 849.33 & 1501.59 & \textbf{\underline{4006.59}} & 3675.40 \\
& BM & 69.31 & 235.46 & 494.08 & \textbf{1450.65} & 1283.83 & 758.16 & 626.94 & 700.60 & 802.93 & 790.78 & 1441.36 & 1677.20 & \textbf{\underline{4508.69}} \\
& BP & 36.85 & 187.29 & 366.83 & \textbf{1400.66} & 1163.14 & 624.77 & 682.54 & 612.99 & 1231.05 & 803.20 & 768.30 & 1752.34 & \textbf{\underline{3649.14}} \\
\cmidrule(lr){1-7} \cmidrule(lr){8-15}
\multirow{15}{*}{\rotatebox{90}{Random access speed (MB/s)}}
& IT & 0.09 & 0.36 & 0.39 & \textbf{1.21} & 1.04 & 1.05 & 0.92 & 1.28 & \textbf{\underline{137.93}} & 1.09 & 19.94 & 5.43 & 48.19 \\
& US & 0.11 & 0.45 & 0.52 & 1.24 & \textbf{1.57} & 1.55 & 1.10 & 1.76 & \textbf{\underline{108.11}} & 1.22 & 22.15 & 6.36 & 57.14 \\
&  ECG & 0.11 & 0.36 & 0.49 & 1.18 & \textbf{1.46} & 1.16 & 1.32 & 1.29 & \textbf{\underline{153.85}} & 1.50 & 30.08 & 3.78 & 50.00 \\
& WD & 0.05 & 0.26 & 0.42 & \textbf{1.23} & 1.01 & 1.12 & 1.02 & 1.21 & \textbf{\underline{135.59}} & 1.28 & 20.05 & 5.19 & 47.06 \\
& AP & 0.10 & 0.39 & 0.57 & \textbf{1.46} & 1.30 & 1.29 & 1.60 & 1.50 & \textbf{\underline{78.43 }}& 1.50 & 19.02 & 5.31 & 53.69 \\
& UK & 0.12 & 0.50 & 0.56 & 1.46 & \textbf{1.63} & 1.20 & 1.05 & 2.01 & \textbf{\underline{131.15}} & 1.26 & 25.94 & 6.26 & 76.92 \\
& GE & 0.12 & 0.42 & 0.51 & 1.38 & 1.46 & 1.42 & 1.02 & 1.65 & \textbf{\underline{186.05}} & 1.18 & 30.75 & 4.23 & 80.00  \\
& LAT & 0.06 & 0.25 & 0.37 & \textbf{1.22} & 1.01 & 1.14 & 1.24 & 0.98 &\textbf{\underline{210.53}} & 0.98 & 22.15 & 5.19 & 76.19 \\
& LON & 0.07 & 0.32 & 0.38 & \textbf{1.22} & 0.99 & 1.16 & 1.26 & 0.97 &\textbf{\underline{210.53}} & 0.99 & 21.07 & 5.12 & 80.01 \\
& DP & 0.07 & 0.33 & 0.37 & \textbf{1.31} & 0.97 & 0.82 & 0.93 & 0.94 & \textbf{\underline{571.43}} & 1.05 & 53.48 & 5.56 & 123.08 \\
& CI & 0.07 & 0.34 & 0.41 & \textbf{1.28} & 1.10 & 1.15 & 0.90 & 1.19 & \textbf{\underline{666.67}} & 1.16 & 57.22 & 6.04 & 140.35 \\
& DU & 0.15 & 0.58 & 0.66 & 1.62 & \textbf{1.95} & 1.38 & 1.38 & 1.92 & \textbf{\underline{363.64}} & 1.50 & 96.39 & 7.87 & 142.46 \\
& BT & 0.03 & 0.18 & 0.39 & \textbf{1.31} & 0.97 & 1.05 & 1.06 & 1.15 & \textbf{\underline{666.67}} & 1.23 & 145.45 & 4.49 & 140.35 \\
& BW & 0.04 & 0.17 & 0.37 & \textbf{1.46} & 0.99 & 0.93 & 1.08 & 1.05 & \textbf{\underline{800.00}} & 1.23 & 131.15 & 4.38 & 148.15 \\
& BM & 0.07 & 0.27 & 0.47 & \textbf{1.40} & 1.25 & 1.05 & 1.06 & 1.18 & \textbf{\underline{533.33}} & 1.13 & 145.45 & 4.61 & 160.00 \\
& BP & 0.04 & 0.22 & 0.38 & \textbf{1.37} & 1.17 & 1.06 & 1.17 & 1.18 & \textbf{\underline{888.89}} & 1.35 & 181.82 & 4.55 & 163.26 \\
\bottomrule
\end{tabular}%
\end{table*}

\section{Related Work}\label{sec:related}

We now review the literature of general- and special-purpose lossless compressors for time series, and of lossy compressors. For the latter, we focus on approaches based on error-bounded functional approximations, which are relevant to our work.

\paragraph{General-purpose lossless compressors}
These compressors are not specifically designed for time series but can be applied to any byte sequence. 
We discuss below the ones based on the LZ77-parsing~\cite{LZ77}, which currently offer the best combination of compression ratio and (de)compression speed.

\brotli~\cite{Alakuijala:2018} relies on a modern variant of the LZ77-parsing of the input file that uses a pseudo-optimal entropy encoder based on second-order context modelling. \xz~\cite{Xz} achieves effective compression by using Markov chain modelling and range coding of the LZ77-parsing. 
\zstd~\cite{Collet:2016} achieves very fast (de)compression speed and good compression ratios via a tabled asymmetric numeral systems encoding. Finally, \lz~\cite{Collet:2013} and \snappy~\cite{Snappy} trade compression effectiveness with speed by adopting a faster  byte-oriented encoding format for the LZ77-parsing. 
Given this plethora of approaches offering a variety of trade-offs, we tested them all in our experiments.

\paragraph{Special-purpose lossless compressors}

Most recent compressors for time series are often based on encoding the result of bitwise XOR operations between close or adjacent floating-point values. Their compression ratio is strongly influenced by data fluctuations: the more severe the fluctuations, the less effective the compression. On the other hand, these algorithms offer very fast (de)compression speeds. %

For instance, Gorilla~\cite{Pelkonen:2015} improves earlier floating-point compressors~\cite{Ratanaworabhan:2006,Lindstrom:2006,Burtscher:2007} by simply computing the XOR between consecutive values of the time series and properly encoding the number of leading zeros and significant bits of the result. 

\chimp~\cite{Liakos:2022} improves both the compression ratio and speed of Gorilla by using different encoding modes based on the number of trailing and leading zeros of the XOR result. \chimpo~\cite{Liakos:2022} and \tsxor~\cite{Bruno:2021} use a window of 128 values to choose the best reference value for the XOR computation: \chimpo uses the value that produces the most trailing zeros, while \tsxor selects the value with the most bits in common.

Elf~\cite{Li:2023} performs an erasing operation on the floats before XORing them, which makes the resulting value more compressible. We do not experiment with Elf because \alp~\cite{Afroozeh:2023} (described next and included in our experiments) was shown to achieve better compression ratios on average, and always faster compression and decompression speeds than Elf.

\alp~\cite{Afroozeh:2023}, unlike the above approaches, does not use XOR operations but rather builds on the idea of encoding a double $x$ via the storage of the significant digits $d$ and an exponent $e$, i.e. $d=\textit{round}(x \cdot 10^e)$, also known as Pseudodecimal Encoding~\cite{BtrBlocks}. It finds a single best exponent for a block of 1024 values and bit-packs the resulting significant digits via the frame-of-reference integer code. Values failing to be losslessly encoded as pseudodecimals are stored uncompressed separately. Further optimisations (such as cutting trailing zeros and using vector instructions) are applied to improve the compression ratio and speed.

BUFF~\cite{BUFF} compresses a float by eliminating the less significant bits based on a given precision, splitting it into the integer and fractional parts, and then compressing the two parts separately with a fixed-length encoding. We do not experiment with BUFF because its average compression ratio on time series was shown to be worse than that of Chimp (which, in turn, is always worse than \ourcompressor in our experiments) and its compression and decompression speeds were shown to be no more than 6$\times$ that of Chimp~\cite{FCBench} (which, in turn, are outmatched by those of ALP by one order of magnitude~\cite{ALP}).

Sprintz~\cite{Blalock:2018} encodes time series using four components: forecasting, bit packing, run-length encoding, and entropy coding. Sprintz focuses on 8- or 16-bit integers, which is a limitation for our datasets with 64-bit data. Also, it was shown to be worse than BUFF in compression ratio and speed~\cite{BUFF}.

There are also floating-point compressors targeted to scientific simulation and observational data, such as fpzip~\cite{fpzip} and ndzip~\cite{ndzip}, but we do not experiment with them since their compression ratios were shown to be poor on time series~\cite{FCBench}.

Concerning the random access operation to time-series values, this is not directly offered by most compressor implementations. Therefore, the typical approach is to compress blocks of the time series separately, and then access a single value by decompressing just the corresponding block.
This is often insufficient to guarantee a reasonable speed and use of computational resources (as our experiments confirm), which is why Brisaboa et al.~\cite{Brisaboa:2013} introduced the Directly Addressable Codes (DAC) scheme that enables fast access to individual values. Given this feature, although designed for generic integer sequences, we included DAC in our experiments.

DACTS~\cite{Brandon:2021} uses Re-Pair~\cite{Larsoon:2000} on top of DAC to better capture repeating patterns. This additional compression step is effective when the time series is highly repetitive but slows down the access time, so DACTS proved to be useful on the so-called industrial time series originating from sensors producing long sequences of constant values. This is a restrictive situation, so we did not experiment with DACTS. 

Titchy~\cite{Vestergaard:2021} focuses on random access to IoT data and relies on a dictionary-based approach combined with a partitioning of the time series into chunks. Each chunk is represented with a pair $\langle \mathit{base}, \mathit{deviation}\rangle$, where $\mathit{base}$ indicates the item of the dictionary to copy, and $\mathit{deviation}$ encodes what makes the chunk slightly different from the other. We could not experiment with Titchy because its source code is not available.%

Finally, \leco~\cite{leco:2024} is a recent proposal (not specific for time series) that, similarly to our \ourcompressor and earlier work~\cite{Boffa:2022talg,Ao:2011}, is based on the idea of lossless compression  via functional approximations and the storage of residuals. \leco uses a learned model to choose a kind of function suitable for the values at hand, and then it uses a heuristic partitioning algorithm that greedily splits fragments (each associated with a function learned through regression methods) and merges neighbouring ones if this improves an estimate of the compression ratio.
Our \ourcompressor, instead, compresses data with different function types, learns each function optimally under a given error bound, and employs a rigorous partitioning algorithm to minimise the actual compression ratio.
These features, together with the more query-efficient compressed layout, make \ourcompressor better than \leco in compression ratio, random access, and decompression speed, as our experiments show.

\paragraph{Functional approximation for lossy compression}
The idea behind lossy functional approximation is to represent a time series as a sequence of functions over time, often under a chosen error metric~\cite{Chiarot:2022survey,Liu:2009}.
In the case of $L_2$-norm, the goal is to minimise the {\em sum of the squared residuals} (i.e. the vertical distances between the true and the approximated values). In the case of the $L_1$-norm, the goal is to minimise the \emph{sum of the absolute values} of the residuals.
The focus of our paper is instead to bound the $L_{\infty}$-norm of the residuals, thus bounding the \emph{maximum} absolute residual.

As discussed in \Cref{sec:background}, the optimal linear-time algorithm to solve this problem using a Piecewise Linear Approximation~(PLA) with the minimum number of segments was first proposed by O'Rourke~\cite{ORourke:1981} (see also \cite{KeoghCHP01segmenting,Dalai:2006,Elmeleegy:2009,Liu:2008,Xie:2014}).
Interestingly, ModelarDB~\cite{ModelarDB} has shown how PLAs can be used in a fully-fledged distributed time series database. ModelarDB could benefit from our results since it computes PLAs via an algorithm~\cite{Elmeleegy:2009} that was shown to use more segments than the optimal one that we use as a baseline~\cite{ORourke:1981,Xie:2014}.
Other works~\cite{Ferragina:2020,KitsiosLPK23} proposed to further compress similar segments, which is a post-processing step that we can apply to our techniques too.

A few works~\cite{Xu:2012,Qi:2015,Eichinger:2015} address the problem of partitioning a time series using nonlinear functions. We compared our approach to the most recent one that employs nonlinear $\varepsilon$\hyp{}approximations~\cite{Xu:2012,Qi:2015}, called Adaptive Approximation~(AA).
AA heuristically partitions a time series using quadratic, linear, and exponential functions. 
However, as our experiments show, despite the use of nonlinear functions, AA is almost always worse than PLA, and in turn worse than our approach. In fact, \ourcompressorlossy finds a better partition of the time series into fragments and selects the best approximating function within a larger set of nonlinear ones.

Finally, we mention HIRE~\cite{HIRE}, which focuses on the distinct problem of constructing a single encoding of a time series that can be decompressed at different $L_{\infty}$ error bounds. Since HIRE relies on piecewise constant approximations as building blocks, we believe it could benefit from our more general nonlinear approximations.

\section{Conclusions and Future Work}\label{sec:conclusions}

We introduced new lossy and lossless compressors that harness the trends and patterns in time series data via a sequence of error-bounded linear and nonlinear approximations of different kinds and shapes. Our approaches experimentally proved to offer new trade-offs in terms of random access speed, decompression speed, and compression ratio compared to existing compressors on 16 diverse time series datasets.

For future work, we suggest further compressing the nonlinear approximation models by exploiting similarities between functions as introduced in~\cite{Ferragina:2020,Kitsios:2023}. 
Another interesting research direction is to exploit the information encoded by the functions to efficiently answer aggregate queries on the time series data. %
Finally, it would be interesting to investigate the impact of our new techniques for computing error-bounded nonlinear approximations in the design of learned data structures~\cite{Kraska:2018,MarcusZK20,kipf:2020,Ferragina:2020,Ferragina2023lemon,Ferragina:2020book,WongkhamLLZLW22,SunZL23,Ferragina2022repetitions,Ferragina:2023string,Amato:2022,Amato:2023,Amato:2023b,AbrarM24}.

{\small\medskip\noindent\textbf{Acknowledgments.}
This work was supported by the European Union -- Horizon 2020 Program under the scheme ``INFRAIA-01-2018-2019 -- Integrating Activities for Advanced Communities'', Grant Agreement n. 871042, ``SoBigData++: European Integrated Infrastructure for Social Mining and Big Data Analytics'' (http://www.sobigdata.eu), by the NextGenerationEU -- National Recovery and Resilience Plan (Piano Nazionale di Ripresa e Resilienza, PNRR) -- Project: ``SoBigData.it - Strengthening the Italian RI for Social Mining and Big Data Analytics'' -- Prot. IR0000013 -- Avviso n. 3264 del 28/12/2021, by the spoke ``FutureHPC \& BigData'' of the ICSC -- Centro Nazionale di Ricerca in High-Performance Computing, Big Data and Quantum Computing funded by European Union -- NextGenerationEU -- PNRR, by Regione Toscana under POR~FSE~2021/27.
}
\balance

\bibliographystyle{IEEEtran}
\bibliography{bibliography}

\vspace{12pt}

\end{document}